\documentclass[twocolumn,10pt]{article}

\usepackage[margin=0.6in,columnsep=22pt]{geometry}
\usepackage{amsmath,amssymb}
\usepackage{amsthm}
\usepackage{booktabs}
\usepackage{enumitem}
\usepackage{natbib}
\usepackage{authblk}
\usepackage{hyperref}
\usepackage{microtype}
\usepackage{graphicx}
\usepackage{xspace}
\usepackage{verbatim}   % provides \begin{comment}...\end{comment}
\usepackage{caption}
\usepackage{titlesec}

\titlespacing*{\section}{0pt}{0.9em}{0.3em}
\titlespacing*{\subsection}{0pt}{0.6em}{0.2em}
\titlespacing*{\paragraph}{0pt}{0.5em}{0.8em}

\newtheorem{theorem}{Theorem}
\newtheorem{proposition}{Proposition}
\newtheorem{lemma}{Lemma}
\newtheorem{corollary}{Corollary}
\newtheorem{remark}{Remark}
\newtheorem{assumption}{Assumption}

\newcommand{\sema}{\hat{H}}
\newcommand{\conf}{\alpha}
\newcommand{\divintra}[1][]{D_{\mathrm{intra}\ifx&#1&\else,#1\fi}}
\newcommand{\Wmat}{W^{*}}
\newcommand{\R}{\mathbb{R}}

\title{When Does Delegation Beat Majority? \\
\large A Delegation-Based Aggregator for Multi-Sample LLM Inference}

\author[1]{Yasushi Sakai}
\author[1]{Allen Song}
\author[1]{Kent Larson}
\affil[1]{MIT Media Lab, Cambridge, MA \\
\texttt{\{yasushis, allen017, kll\}@media.mit.edu}}
\date{}
\begin{document}
\twocolumn[
  \begin{@twocolumnfalse}
    \maketitle
    \begin{center}
    \begin{minipage}{0.78\linewidth}
    \begin{abstract}
      \noindent
      Majority voting is the default unsupervised aggregator for
      multi-sample LLM inference, but it discards two signals: 
      within-group answer entropy and between-group reasoning
      geometry. We aggregate by delegation instead (Propagational Proxy
      Voting, PPV; \citealp{sakai2025ppv}): each group of samples keeps
      weight on its own answer in proportion to its entropy-based
      confidence (\textsc{When}) and routes the rest to peers by
      reasoning-embedding similarity (\textsc{Whom}); the stationary
      distribution of the resulting delegation matrix picks the consensus
      answer. This requires neither gold labels nor training.
      On MMLU-Pro with $128$ samples per question, delegation beats majority by $+1.5$\,pp overall and $+2.24$\,pp on non-trivial questions (McNemar
      $p \approx 1.0 \times 10^{-14}$, $n = 8{,}099$), overturning wrong
      majorities whose answer cluster is geometrically incoherent while
      the correct minority is tight. We then characterize exactly when
      delegation overturns majority: a two-option model gives a
      closed-form flip condition on each option's confidence and the
      weight it routes to the other, with a do-no-harm corollary for
      near-unanimous questions.
      The condition calls the realized winner on $96.5\%$ of non-trivial
      questions, and its predicted mass gap tracks the realized gap at
      $r = 0.97$. We did not find any other unsupervised ensemble methods that
      close the oracle gap.
    \end{abstract}
    \end{minipage}
    \end{center}
    \vspace{1em}
  \end{@twocolumnfalse}
]

% =============================================================================
\section{Introduction}
% =============================================================================

Sampling an LLM many times and aggregating the answers, a procedure
known as \emph{self-consistency} \citep{wang2023selfconsistency}, is now
standard practice for reasoning tasks. The aggregator is almost always majority vote over the parsed
answers: simple, model-agnostic, label-free, and a strong baseline. Yet
each sample carries two signals beyond its parsed answer that
majority discards:
\begin{enumerate}[leftmargin=1.5em, topsep=2pt, itemsep=2pt]
\item \textbf{Letter-level uncertainty.} A group of $k$ samples that all
  reach the same letter is more informative than $k$ samples split
  $\lceil k/2 \rceil$ to $\lfloor k/2 \rfloor$. Letter entropy over a small
  group is a free per-group confidence signal, and it is precisely the
  MCQ-degenerate case of semantic
  entropy \citep{kuhn2023semantic, farquhar2024hallucinations}.
\item \textbf{Reasoning geometry.} Embedding each sample's reasoning text
  gives a high-dimensional position in semantic space. Two groups that
  pick the same letter via similar reasoning sit close in that space; two
  groups that pick the same letter via unrelated reasoning sit far apart.
  Majority does not take this into account.
\end{enumerate}

\begin{figure}[!t]
\centering
\includegraphics[width=\linewidth]{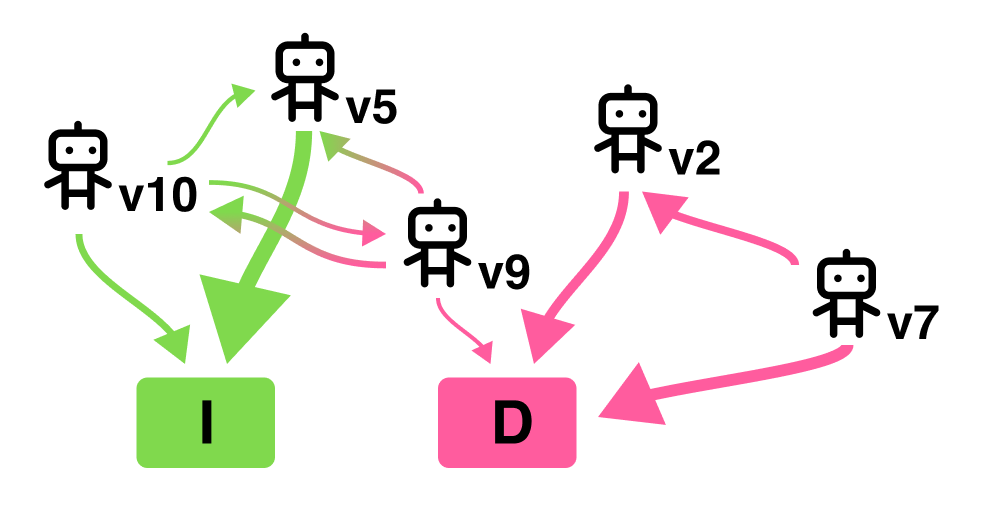}
\caption{\textbf{Simplified network of direct voting and delegation 
from problem \texttt{philosophy\_314}.} Each bot is a voter (a group of LLM
samples). It keeps some weight on its own answer letter
and routes the rest to peers whose reasoning-embedding aligns. The
full graph has $16$ voters and majority picks the wrong letter $10$
to $6$; we draw $5$ ($3$ majority + $2$ minority) for legibility. A
coherent minority cluster plus a defecting majority voter (\,$v_9$\,)
absorb enough redirected mass to flip the consensus. The green letter
\texttt{I} stands in for the gold answer which PPV picked; the red
\texttt{D} is the letter the majority collectively mispicked.}
\label{fig:ppv-teaser}
\end{figure}

\paragraph{Approach.} We treat each group of samples as a \emph{delegate}
in a voting graph and aggregate via Propagational Proxy Voting (PPV;
\citealp{sakai2025ppv}), a liquid-democracy mechanism that propagates
voting mass through an absorbing Markov chain. PPV exposes two levers
per delegate, one for each of the discarded signals above:
\begin{itemize}[leftmargin=1.5em, topsep=2pt, itemsep=2pt]
\item \textsc{When}: how much weight a delegate keeps on its own pick.
  We drive it with the group's letter entropy: low entropy
  $\Rightarrow$ high self-weight.
\item \textsc{Whom}: how the remaining weight is split across peer
  delegates. We drive it with per-question-centered embedding cosine:
  delegate to peers whose reasoning aligns.
\end{itemize}
The result is a parameter-free, label-free aggregator, once embeddings
are precomputed.

\paragraph{Contributions.}
\begin{enumerate}[leftmargin=2.1em, topsep=2pt, itemsep=2pt, label=(\roman*)]
\item A concrete parameterization of PPV from per-LLM-sample
  signals: \textsc{When} from letter-level semantic entropy, and
  \textsc{Whom} from per-question-centered embedding cosine. This yields
  a recipe for using delegation networks as an unsupervised aggregator
  over sampled generations. Per-question centering is essential: raw
  same-question cosines sit in $[+0.88, +0.99]$; centering exposes a
  discriminative geometry spanning $[-0.68, +0.64]$.
\item A large-scale empirical study on the full $12{,}032$-question
  MMLU-Pro test split: PPV with $\conf = 1 - \sema$ achieves $42.2\%$
  versus majority's $40.7\%$, and $30.2\%$ versus $28.0\%$ on the
  $8{,}099$ non-trivial questions ($+2.24$\,pp, paired McNemar
  $p \approx 10^{-14}$).
\item A \textsc{When}/\textsc{Whom} decomposition that isolates the
  lever: the entire gain comes from \textsc{When}. Explicit peer-quality
  multipliers in \textsc{Whom} provide no lift, and in $11$ of $25$
  tested configurations they actively hurt, because PPV's multi-hop
  propagation does implicit \emph{quality laundering}.
\item Negative results that constrain the design space for unsupervised
  LLM aggregators: $P(\text{True})$ \citep{kadavath2022know} has
  area under the ROC curve (AUROC) $0.47$, which is anti-correlated with
  correctness; CoCoA-style products \citep{becker2024cocoa} are dragged
  down by it; and the unsupervised mode-selection ensembles we piloted
  do not beat the best single mode, leaving per-question polarity
  selection as the open problem where supervision plausibly helps.
\item An exact two-block characterization of when delegation overturns
  majority (Section~\ref{sec:theory}): a closed-form flip condition on
  confidence and cross-cluster leakage, with a conditional do-no-harm
  corollary, validated against the full mechanism.
\end{enumerate}

% =============================================================================
\section{Related Work}
\label{sec:related}
% =============================================================================

Our work sits at the intersection of six threads of prior work.
We survey each in turn and situate our contribution.

% ---------------------------------------------------------------------------
\subsection{Self-consistency and multi-sample aggregation}
% ---------------------------------------------------------------------------

\citet{wang2023selfconsistency} established sample-and-vote as the default
unsupervised aggregator for chain-of-thought reasoning.  The aggregator is
plain plurality: generate many responses, extract the answer string from
each, and return the most frequent one.  Subsequent work has explored
alternatives within this regime.  Verifier-reranked voting
\citep{cobbe2021training} scores candidates with a trained reward model
before selecting.  There are generative approaches like Universal Self-Consistency (USC;
\citealt{chen2023universal}) which prompts the LLM itself to nominate the most
coherent candidate.  Ranked voting methods such as instant-runoff voting, 
Borda count, mean reciprocal rank have recently been applied to LLM
self-consistency, yielding modest gains over plurality \citep{chen2025ranked}.

The most directly concurrent work to ours is \citet{pan2025beyond}, who
replace majority with aggregators that exploit first- and second-order
correlations among model responses, evaluating on MMLU and UltraFeedback;
we differ in using a single model's repeated samples (not a panel of models),
and in grounding the aggregator in a formal delegation mechanism with an
explicit per-voter confidence parameterization.

Adaptive sampling methods reduce the sample budget without sacrificing
accuracy: RASC \citep{wan2024rasc} trains a CoT-quality scoring function
for early stopping and score-weighted voting; \citet{aggarwal2023let}
explores heuristic stopping rules on the sample set.
\citet{cordero2025certified} provide a theoretical foundation, deriving
finite-sample concentration bounds that quantify how reliably majority vote
recovers the mode of the model's distribution, and introduce the Martingale
Majority Certificate as a sequential stopping rule.  These results
characterize the regime our aggregator operates in but do not address
\emph{which} aggregation function to use once a sample budget is fixed.

% ---------------------------------------------------------------------------
\subsection{Test-time compute scaling}
% ---------------------------------------------------------------------------

Test-time compute (TTC) allocates additional inference compute to improve
output quality.  \citet{snell2024scaling} survey the space, classifying
methods into parallel (sample and aggregate) and sequential (iterative
refinement, tree search).  Best-of-N with a verifier
\citep{cobbe2021training,lightman2023let} is the standard parallel
baseline.  Beam search and Monte Carlo Tree Search \citep{yao2023tree}
explore the sequential branch.  Our work is parallel and unsupervised:
we draw $128$ samples and apply a richer aggregation function, with no
additional rollouts and no verifier.

\citet{muennighoff2025s1} show that a simple budget-forcing approach
(``wait'' tokens) transfers reasoning compute into longer chains.
\citet{sequentialedge2025} argue that inverse-entropy–weighted voting over
sequentially refined outputs outperforms parallel majority at matched
compute, connecting entropy-based weighting to the sequential paradigm.
Our finding that $\conf = 1 - \sema$ (inverse entropy as confidence)
delivers the full gain in the parallel regime is broadly consistent with
the entropy-weighting intuition, while demonstrating it within the
PPV delegation framework.

% ---------------------------------------------------------------------------
\subsection{Semantic Entropy (SE) and uncertainty quantification}
% ---------------------------------------------------------------------------

\citet{kuhn2023semantic} introduced semantic entropy: cluster generations
by meaning equivalence (via NLI), then take the entropy of the cluster
distribution.  \citet{farquhar2024hallucinations} extended the method for
hallucination detection at scale (Nature 2024), showing that semantic
entropy is a reliable unsupervised signal for factual reliability.  For
multiple-choice questions the NLI clustering degenerates to grouping by the
extracted letter, which is the form we use; we additionally apply the
Miller--Madow bias correction \citep{miller1955note}.

Several recent papers extend or approximate semantic entropy.
\citet{kl2024sep} propose Semantic Entropy Probes (SEPs), which
approximate semantic entropy from single-forward-pass hidden states,
reducing the 5--10$\times$ inference overhead.  Kernel Language Entropy
(KLE; \citealt{duan2024kle}) generalizes SE to a kernel-based uncertainty
measure that captures both intra-cluster spread and inter-cluster distance
without hard partitions; \citet{nguyen2025snne} (ACL 2025) independently
make a similar argument and propose SNNE, a nearest-neighbour entropy
estimator with provable generalization over SE. We use letter-level semantic
entropy which is a low-overhead instantiation of this family tailored to
multiple-choice; the centering transformation we apply to embeddings before
computing cosines is in spirit similar to KLE's inter-cluster term.  A
statistically consistent estimator of semantic uncertainty for open-ended
generation is studied by \citet{nikitin2024consistent}, with formal
guarantees on convergence.

% ---------------------------------------------------------------------------
\subsection{Confidence calibration and self-verification}
% ---------------------------------------------------------------------------

\citet{kadavath2022know} showed that LLMs can be prompted to estimate
$P(\text{True})$ for their own outputs.  CoCoA
\citep{becker2024cocoa} combines $P(\text{True})$ with semantic entropy
multiplicatively.  We evaluate $P(\text{True})$ on our setup and find it
anti-correlated with correctness (AUROC $0.47$); CoCoA-style products
inherit the anti-correlation and underperform $\sema$ alone
(\S\ref{sec:negatives}).

The broader calibration literature makes such failures unsurprising.  \citet{guo2017calibration}
document that standard training produces systematically over-confident
models; temperature scaling corrects marginal calibration but not
group-level calibration.  \citet{zhou2024calibrating} show that RL fine-tuning
(DPO, PPO, GRPO) degrades calibration by exploiting reward advantage
weighting; post-RL SFT restores it.  The Qwen3-1.7B model we use is
trained with RL, which offers a mechanistic explanation for why its
$P(\text{True})$ is anti-correlated with correctness in our setting.
Graph-based confidence calibration \citep{li2024graph} uses similarity
graphs over multiple responses and learns to correct miscalibration,
which is a supervised approach that would require gold labels unavailable
in our unsupervised regime.  \citet{lin2024generating} study verbalized
confidence and its calibration properties.  The overall lesson across this
literature is that reliable confidence signals for small RL-trained models
at high temperature cannot be taken for granted; our negative result
on $P(\text{True})$ is consistent with this.

% ---------------------------------------------------------------------------
\subsection{Liquid democracy and propagational proxy voting}
% ---------------------------------------------------------------------------

Liquid democracy \citep{ford2002delegative} is a voting framework in which
agents may delegate their votes transitively.  Its properties have been
studied theoretically \citep{christoff2017preference,brill2018interactive},
including the Condorcet-jury analysis of when delegation helps versus hurts
relative to direct voting \citep{kahng2018liquid}, the algorithmic
perspective on optimal delegation \citep{brill2022liquid}, and the
game-theoretic analysis of rational delegation \citep{bloembergen2019rational}.
Recent computational social-choice work continues to refine the picture:
\citet{alouf2024controlling} (IJCAI 2024) study manipulation of delegation
graphs; \citet{alouf2025cost} (AAAI 2025) analyze the welfare cost of
liquid mechanisms; and \citet{bersetche2025generalizing} (IJGT 2025)
generalize liquid democracy to multi-agent settings with equilibrium
analysis.

\citet{sakai2025ppv} introduce PPV, an absorbing-Markov-chain formulation
that admits \emph{split} delegation (each voter can distribute its budget
across multiple peers and one policy simultaneously).  To our knowledge the
present work is the first application of liquid democracy, or indeed of
any delegation-graph mechanism, to multi-sample LLM aggregation.  The
conceptual mapping is clean: each group of samples is a voter, the
delegation budget encodes trust in peers calibrated by embedding cosine,
and the stationary distribution of the chain resolves the consensus.

% ---------------------------------------------------------------------------
\subsection{Multi-agent debate and ensemble consensus}
% ---------------------------------------------------------------------------

A parallel line of work, multi-agent debate (MAD), aggregates outputs
by having multiple LLM \emph{instances} debate iteratively
\citep{du2023improving, liang2024encouraging}.
\citet{khan2024debating} show that debate between models can surface
factual errors that a single model would propagate.  The key differences
from our setting are: (i) MAD involves multiple distinct models or
independently prompted instances that communicate across rounds, whereas
our setting draws repeated samples from a fixed temperature distribution of
one model; (ii) MAD aggregation is sequential (each round conditions on
prior outputs), whereas ours is parallel; (iii) MAD is compute-intensive
and can exhibit sycophancy, with agents converging to a wrong consensus
under social pressure \citep{consensagent2025}.

Mixture-of-agents approaches \citep{wang2024moa} use heterogeneous models
as an ensemble; \citet{zhao2024llmcouncil} run a council of LLMs that
collectively author, take, and grade a benchmark, ranking each other in
a democratic fashion. These methods lie in a complementary regime
(multiple models, iterative, often supervised in some component) to our
single-model, single-pass, fully unsupervised approach.  Recent work
applying social choice formalisms to multi-agent debate
\citep{fromdebatetodecision2025} asks when it is \emph{safe} to commit to
a debate outcome, which is a question about confidence rather than
aggregation function, but one that connects to our concern about polarity
mismatch.

Our work also differs from \citet{pan2025beyond} and concurrent LLM-panel
aggregation methods in that we treat the \emph{reasoning text} as a
first-class signal (via embedding geometry) rather than only the final
answer distribution.  This reasoning-geometry lever is invisible to methods
that pool only extracted answers.

% =============================================================================
\section{Preliminaries: Propagational Proxy Voting}
\label{sec:prelim}
% =============================================================================

We summarize the PPV machinery used; for proofs and the general formulation
see \citet{sakai2025ppv}.

\paragraph{Setup.} Fix a set of voters $N = \{d_1, \dots, d_n\}$ and a set
of policies (possible answers) $P$. PPV operates on a
column-stochastic \emph{voting matrix}:
\begin{equation}
V \;=\; \begin{bmatrix} V_{d \leftarrow d} & \mathbf{0} \\
V_{p \leftarrow d} & I_{|P|} \end{bmatrix}, \quad
\sum_{i} V_{i,j} = 1 \;\; \forall j,
\label{eq:V}
\end{equation}
where $V_{d \leftarrow d} \in \R^{n \times n}$ is the delegate-to-delegate
sub-block (with zero diagonal), $V_{p \leftarrow d} \in \R^{|P| \times n}$ is
the delegate-to-policy sub-block, and $I_{|P|}$ makes each policy an absorbing
state. Column $j$ encodes the outgoing distribution of delegate $d_j$: how
its unit of voting mass splits between its own pick (a policy) and its
peers (other delegates).

\paragraph{Consensus via the limit matrix.}
\citet[Theorem~IV.3]{sakai2025ppv} prove that the limit
$\Wmat = \lim_{x \to \infty} V^{x}$ exists and is computable by repeated
squaring \citep[Corollary~IV.4]{sakai2025ppv}. The columns of $\Wmat$ live
entirely on the policy block: each delegate's mass is fully absorbed. The
PPV \textbf{consensus winner} is
\begin{equation}
\hat{p} \;=\; \arg\max_{p \in P} \; \sum_{j \in N} \Wmat_{p, j}.
\label{eq:winner}
\end{equation}

PPV reduces to several familiar mechanisms in limits: if every column has
$V_{\pi_j, j} = 1$ (each delegate keeps everything), PPV collapses to
majority vote weighted by the column distribution; if some column has full
mass on a single peer, that peer absorbs the delegator's vote entirely
(classical proxy voting).

% =============================================================================
\section{Method}
\label{sec:method}
% =============================================================================

We instantiate PPV for unsupervised LLM aggregation in four steps:
sampling, signal extraction, matrix construction, and propagation.

\subsection{Sampling and partitioning}

For each question $q$, the solver generates $S = 128$ chain-of-thought
samples, each terminating in an extracted answer letter $\ell_c \in
\{A, \dots, J, \varnothing\}$ ($\varnothing$ for parsing failure). We
partition the $128$ samples deterministically into $n = 16$ groups of
$g = 8$. Each group becomes a \emph{voter} (delegate) in PPV's terminology;
its pick $\pi_j$ is the majority letter among its $8$ samples
($\pi_j = \texttt{Z}$ if the majority is $\varnothing$). The set of
policies $P$ is the set of letters actually picked. The partition is
fixed (not learned); $128/16$ gives each group enough samples for a stable
entropy estimate while leaving enough voters for meaningful delegation.

\subsection{Per-voter signals}

\paragraph{Letter-level semantic entropy.} For voter $j$ with letter counts
$\{c_\ell\}_{\ell \in \mathcal{L}_j}$ over its $8$ samples
($\mathcal{L}_j$ = letters observed, $K_j = |\mathcal{L}_j|$), the
Miller--Madow corrected and $\log g$-normalized entropy is
\begin{equation}
\sema_j \;=\;
\frac{- \sum_{\ell \in \mathcal{L}_j} \frac{c_\ell}{g}
\log \frac{c_\ell}{g} \;+\; \frac{K_j - 1}{2g}}{\log g}.
\label{eq:sema}
\end{equation}
$\sema_j = 0$ when all $8$ samples agree on one letter; values approach $1$
when the $8$ samples spread evenly across letters. The Miller--Madow
correction $\frac{K_j-1}{2g}$ \citep{miller1955note} compensates for the
downward bias of plug-in entropy at $g = 8$; for fully spread distributions
the correction can push the normalized value mildly above $1$, in which
case we clip downstream confidence values to $[0, 1]$.

\paragraph{Reasoning embeddings and per-question centering.} Each sample's
reasoning text is embedded by Qwen3-Embedding-8B into a $4096$-dimensional
unit vector $e_c$. Let $\bar e^{(q)} = \frac{1}{S} \sum_{c=1}^{S} e_c$
denote the per-question centroid over all $S = 128$ samples. We center and
renormalize:
\begin{equation}
\tilde e_c \;=\; \frac{e_c - \bar e^{(q)}}
                      {\| e_c - \bar e^{(q)} \|_2}, \quad
\bar e_j \;=\; \frac{1}{g}\sum_{c \in \text{group}_j} \tilde e_c,
\end{equation}
and define the inter-voter cosine matrix
$\cos_{ij} = \langle \bar e_i / \|\bar e_i\|, \bar e_j / \|\bar e_j\| \rangle$.

Centering is the difference between an informative geometry and noise. With uncentered embeddings,
off-diagonal cosines on a single question lie in $[+0.88, +0.99]$ with standard deviation $\approx 0.025$: every voter's
reasoning embedding is dominated by the shared question content. Subtracting
the per-question centroid removes that common component and yields cosines
in $[-0.68, +0.64]$ with standard deviation $\approx 0.32$.

\paragraph{Intra-group reasoning diversity.} Within voter $j$'s $8$ samples,
the average pairwise centered cosine quantifies how varied its internal
reasoning is:
\begin{equation}
\divintra[j] \;=\; \frac{1}{2}
\left(1 - \frac{1}{g(g-1)} \!\!\sum_{\substack{a, b \in \text{group}_j \\ a \ne b}}\!\!
\langle \tilde e_a, \tilde e_b \rangle \right) \in [0, 1].
\end{equation}
A voter with $\divintra[j]$ near $0$ has $8$ near-duplicate samples
(templated reasoning); near $0.5$ has $8$ mutually orthogonal samples
(independent reasoning paths).

\subsection{Constructing the voting matrix}

For each voter $j$ we choose a \emph{confidence}
$\conf_j = f(\sema_j, \divintra[j]) \in [0, 1]$ from one of five modes
(Table~\ref{tab:modes}). Column $j$ of $V$ is then
\begin{equation}
\begin{aligned}
V_{\pi_j, j} &= \conf_j, \\
V_{i, j} &= (1 - \conf_j) \cdot \frac{\max(\cos_{ij}, 0)}
{\sum_{\ell \ne j} \max(\cos_{\ell j}, 0)} \;\; \forall\, i \ne j,
\end{aligned}
\label{eq:column}
\end{equation}
with $V_{j, j} = 0$ and $V_{p, j} = 0$ for $p \ne \pi_j$. The \textsc{When}
mass $\conf_j$ goes into the policy block ($V_{p \leftarrow d}$); the
\textsc{Whom} mass $1 - \conf_j$ goes into the delegate block
($V_{d \leftarrow d}$). Clipping negative cosines to zero encodes ``do not
delegate to peers whose reasoning is anti-aligned with mine.'' If every
$\max(\cos_{\ell j}, 0) = 0$ (no positively-aligned peer), the peer budget
is split uniformly over the other $n - 1$ voters, keeping the column
stochastic.

\begin{table}[t]
\centering
\caption{Confidence modes. $s = \sema$, $d = \divintra$. All clipped to $[0, 1]$.}
\label{tab:modes}
\small
\begin{tabular}{ll}
\toprule
\textbf{Mode} & $\conf = f(s, d)$ \\
\midrule
\texttt{confidence}        & $1 - s$ \\
\texttt{inverted}          & $s$ \\
\texttt{confidence\_x\_div}& $(1-s) \, d$ \\
\texttt{inverted\_x\_div}            & $s \, d$ \\
\texttt{div}               & $d$ \\
\bottomrule
\end{tabular}
\end{table}

\subsection{Propagation and decision}

We compute $\Wmat$ by repeated squaring of $V$ and read off the consensus
winner via Equation~\ref{eq:winner}. 

% =============================================================================
\section{Experiments}
\label{sec:experiments}
% =============================================================================

\paragraph{Setup.} Solver: Qwen3-1.7B at temperature $1.5$, $128$
chain-of-thought samples per question. Benchmark: MMLU-Pro test split
\citep{wang2024mmlupro}, all $14$ subjects, $12{,}032$ questions.
Embeddings: Qwen3-Embedding-8B \citep{qwen2025embedding}, $4096$-dim,
L2-normalized, stored as fp16 memmap. This equates to a total generation
of $\approx 1.54$M reasoning trajectories.

\paragraph{Baselines.}
\begin{itemize}[leftmargin=1.2em, topsep=2pt, itemsep=2pt]
\item \textbf{Majority}: top letter across $16$ voter picks, $\varnothing$
  excluded; ties counted incorrect.
\item \textbf{Best dictator}: the single voter index with highest accuracy
  across the dataset, applied to every question.
\item \textbf{Oracle (pass@16)}: correct iff any of the $16$ voters picked
  gold. Ceiling for any $16$-voter aggregator.
\end{itemize}

\paragraph{Main results.} Table~\ref{tab:main} shows that PPV with
$\conf = 1 - \sema$ beats majority by $+1.50$\,pp overall and $+2.24$\,pp
on the non-trivial subset. The paired McNemar test on the non-trivial
subset (Table~\ref{tab:mcnemar}) reports $366$ PPV wins versus $185$
majority wins out of $8{,}099$ non trivial questions, $p \approx 1.0 \times 10^{-14}$.

\begin{table}[t]
\centering
\caption{Accuracy on MMLU-Pro test (Qwen3-1.7B, $128$ samples/question).
A question is \emph{trivial} when its top letter receives $\ge 12/16$
votes; on such questions all confidence modes agree with majority by
construction. We report the full set and the non-trivial remainder
($n = 8{,}099$ of $12{,}032$), where the differentiating signal lives.
}
\label{tab:main}
\small
\begin{tabular}{lrrr}
\toprule
\textbf{Method} & \textbf{All Q} & \textbf{Non-trivial} & \textbf{$\Delta$ maj.} \\
\midrule
best dictator                    & 35.78\% & --      & --          \\
majority                         & 40.71\% & 28.00\% & --          \\
PPV \texttt{inverted\_x\_div}        & 41.71\% & 29.52\% & $+1.52$\,pp \\
PPV \texttt{div}           & 41.92\% & 29.83\% & $+1.83$\,pp \\
\textbf{PPV \texttt{confidence}} & \textbf{42.21\%} & \textbf{30.24\%} & \textbf{$+2.24$\,pp} \\
\midrule
oracle (pass@16)                 & 44.68\% & 33.93\% & $+5.93$\,pp \\
\bottomrule
\end{tabular}
\end{table}

\begin{figure}[t]
\centering
\includegraphics[width=\linewidth]{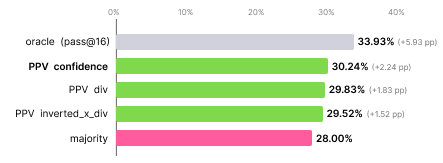}
\caption{Accuracy on the $8{,}099$ non-trivial MMLU-Pro questions.
Bars zoomed to highlight the $\sim$6\,pp gap between majority and oracle.
PPV (\texttt{confidence}) closes $38\%$ of that gap unsupervised.}
\label{fig:results}
\end{figure}

\begin{table}[t]
\centering
\caption{Paired McNemar test, \texttt{confidence} vs.\ majority on
the $8{,}099$ non-trivial questions.}
\label{tab:mcnemar}
\small
\begin{tabular}{lr}
\toprule
PPV correct, majority wrong  & $366$ \\
Majority correct, PPV wrong  & $185$ \\
Net                          & $+181$ \\
Two-sided exact binomial $p$ & $\approx 1.0 \times 10^{-14}$ \\
\bottomrule
\end{tabular}
\end{table}

\paragraph{Per-disagreement precision.} On the non-trivial subset,
\texttt{div} wins $187/(187+39) = 82.7\%$ of disagreements,
\texttt{inverted\_x\_div} wins $68.9\%$, and \texttt{confidence} wins
$66.4\%$ (Figure~\ref{fig:disagreement}). The aggregate-accuracy winner
(\texttt{confidence}) and the per-disagreement-precision winner
(\texttt{div}) are different modes: the former disagrees with majority
more often and is right less often per disagreement; the latter
disagrees less often but more reliably.

\begin{figure}[t]
\centering
\includegraphics[width=\linewidth]{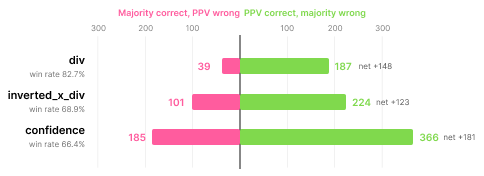}
\caption{Where the disagreements land on the non-trivial subset. Left
(pink) = majority correct, PPV wrong; right (green) = PPV correct,
majority wrong. \texttt{div} fires rarely but is right $82.7\%$ of the
time; \texttt{confidence} fires more and rescues the most absolute
questions ($+181$ net).}
\label{fig:disagreement}
\end{figure}

\paragraph{No single mode dominates.} The five modes form a Pareto front,
not a ranking. \texttt{confidence} maximizes overall accuracy;
\texttt{div} maximizes per-disagreement precision; \texttt{inverted\_x\_div}
maximizes accuracy on the inverted-polarity ``confidently wrong'' family.
The modes rescue overlapping but distinct sets of questions.

% =============================================================================
\section{Mechanism}
\label{sec:mechanism}
% =============================================================================

We unpack \emph{how} PPV delivers the gain through a worked example, the
mirrored failure case, and a \textsc{When}/\textsc{Whom} ablation.

\subsection{A clear majority overturned: \texttt{philosophy\_314}}
\label{sec:philosophy314}

Question \texttt{philosophy\_314} (gold $=$ I) is a $10$--$6$ majority for
the wrong letter: $10$ voters pick \texttt{D}, $6$ pick \texttt{I}. This is
not a tiebreak, since D wins majority with a $4$-vote margin over the
runner-up. PPV with \texttt{confidence} nonetheless resolves the question to
\texttt{I}. The flip is driven by \textsc{Whom}, not \textsc{When}.

\paragraph{Own-pick weight (\textsc{When}) is nearly a wash.}
D-pickers have mean $\sema = 0.713$ and I-pickers $0.673$, a gap of only
$0.04$. Under $\conf = 1 - \sema$, the first-iteration self-mass favors
\emph{D}: $\sum_{j \in D} \conf_j = 2.87$ vs $\sum_{j \in I} \conf_j = 1.96$.
Entropy alone would keep D ahead.

\paragraph{Reasoning geometry (\textsc{Whom}) is decisive.} In centered
embedding space, the two clusters look qualitatively different
(Figure~\ref{fig:phil314-pca}). The $10$ D-voters reach \texttt{D} via
\emph{unrelated} reasoning paths, with mean within-D cosine $+0.050$, a
near-orthogonal cloud. The $6$ I-voters reach \texttt{I} via tightly
similar reasoning, with mean within-I cosine $+0.519$. The cross-cluster
mean cosine is $+0.197$ --- \emph{higher} than within-D: a typical
D-voter's reasoning sits closer to the I-cluster than to its fellow
D-voters. The clipped-cosine peer weights inherit this asymmetry:
D-pickers route on average $53.4\%$ of their peer budget toward the
I-cluster, even though it holds only $6$ of their $15$ peers, because the
within-D weights are too small to compete. I-pickers reciprocate less
($48.1\%$ of their budget crosses back) while keeping more own-pick mass,
so the net delegation flow points from D to I. The $16 \times 16$ initial
voting matrix $V$ (Figure~\ref{fig:phil314-V}) makes this visible: the
I-cluster's rows ($d_3, d_5, d_6, d_{10}, d_{11}, d_{14}$) absorb the
bulk of the off-diagonal mass from \emph{both} blocks' columns.

\begin{table*}[t]
\centering
\caption{Per-voter signals on \texttt{philosophy\_314}. $\conf = 1 - \sema$
is the \textsc{When} own-pick weight; $1 - \conf$ is the peer budget split
via clipped centered cosine. D-pickers (left of separator) vs.\ I-pickers
(right).}
\label{tab:phil314-voters}
\footnotesize
\setlength{\tabcolsep}{4.2pt}
\begin{tabular}{l|cccccccccc|cccccc}
\toprule
voter        & $d_0$ & $d_1$ & $d_2$  & $d_4$ & $d_7$ & $d_8$  & $d_9$ & $d_{12}$ & $d_{13}$ & $d_{15}$ & $d_3$ & $d_5$  & $d_6$ & $d_{10}$ & $d_{11}$ & $d_{14}$ \\
pick         & D     & D     & D      & D     & D     & D      & D     & D        & D        & D        & I     & I      & I     & I        & I        & I        \\
\midrule
$\sema_{mm}$ & 0.839 & 0.839 & 0.606  & 0.839 & 0.529 & 0.694  & 0.839 & 0.694    & 0.725    & 0.529    & 0.725 & 0.348  & 0.725 & 0.725    & 0.725    & 0.787    \\
$\divintra$  & 0.517 & 0.507 & 0.506  & 0.504 & 0.504 & 0.502  & 0.499 & 0.490    & 0.486    & 0.485    & 0.488 & 0.503  & 0.500 & 0.493    & 0.487    & 0.479    \\
$\conf$      & 0.161 & 0.161 & 0.394  & 0.161 & 0.471 & 0.306  & 0.161 & 0.306    & 0.275    & 0.471    & 0.275 & \textbf{0.652} & 0.275 & 0.275 & 0.275 & 0.213 \\
\bottomrule
\end{tabular}
\end{table*}

\begin{figure}[t]
\centering
\includegraphics[width=\linewidth]{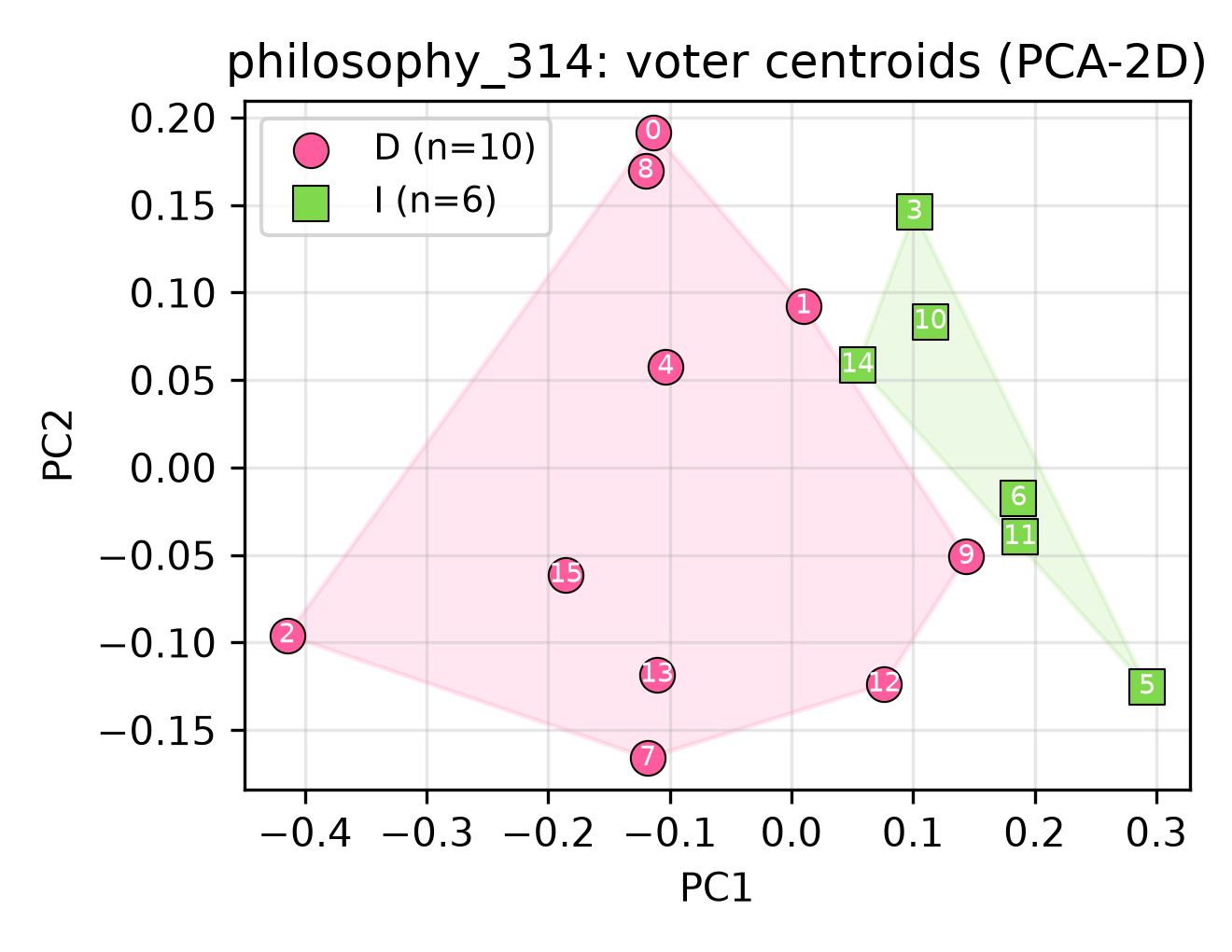}
\caption{PCA-2D of the $16$ per-voter centroids on
\texttt{philosophy\_314} after per-question centering; shaded regions
are the convex hulls of the two clusters (green $=$ gold letter
\texttt{I}, pink $=$ the mispicked \texttt{D}). The $10$ D-pickers form
a scattered cloud (mean within-cluster cosine $+0.050$) whose hull
spans the plane; the $6$ I-pickers form a tight cluster ($+0.519$).
The geometric incoherence of the majority is what PPV exploits.}
\label{fig:phil314-pca}
\end{figure}

\begin{figure}[t]
\centering
\includegraphics[width=\linewidth]{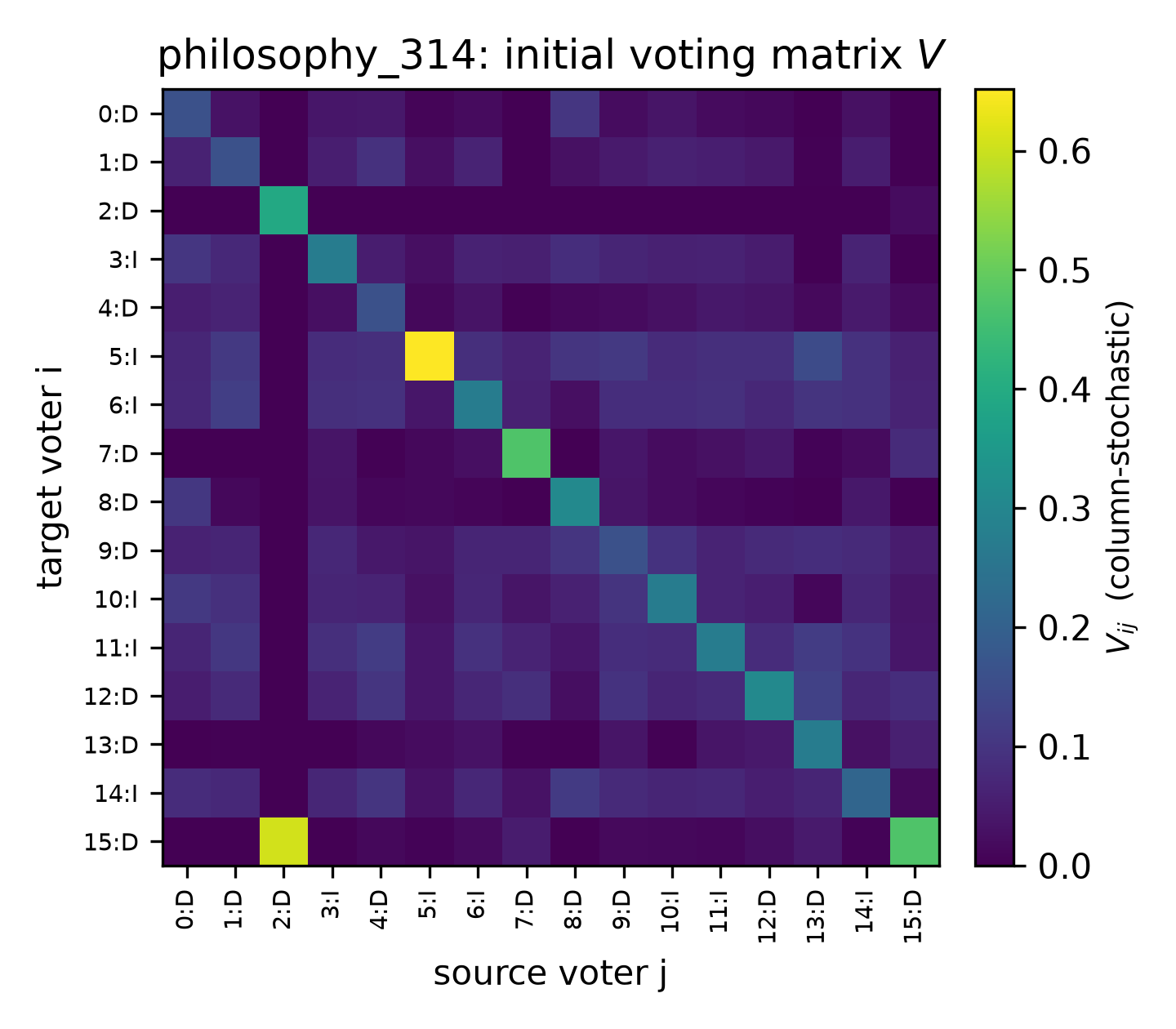}
\caption{Initial column-stochastic voting matrix $V$ for
\texttt{philosophy\_314}. Diagonal $=$ \textsc{When} own-pick weight
$\conf_j = 1 - \sema_j$; column off-diagonals $=$ \textsc{Whom} peer
split by clipped centered cosine. The I-cluster rows absorb the bulk
of the off-diagonal mass from both blocks' columns; within-D weights
are too small to compete.}
\label{fig:phil314-V}
\end{figure}

\paragraph{Propagation.} After repeated squaring of $V$, the stationary
distribution lands at
\[
\textstyle\sum_j W_{I, j} = 8.94, \quad \sum_j W_{D, j} = 7.06
\;\Rightarrow\; \hat{p} = \texttt{I} \;\;\text{(gold)}.
\]
A clear $10$--$6$ majority overturned because the minority's reasoning is
coherent and the majority's is not. \textsc{Whom} did the work that
\textsc{When} alone could not.

\paragraph{When entropy carries the day: \texttt{economics\_124}.} Not all
PPV wins look like \texttt{philosophy\_314}. On \texttt{economics\_124}
(gold $=$ \texttt{J}, picks split $8$/$8$ between \texttt{I} and
\texttt{J}), the entropy gap is the visible lever: J-pickers have mean
$\sema = 0.536$ vs I-pickers' $0.623$, and the cosine geometry is only
mildly assortative (within-J $+0.276$, within-I $+0.233$, cross
$+0.265$). The two signals each do a little work and compound to
$\sum_j W_{J,j} = 8.69$ vs $7.31$, flipping a tiebreak. Across the
$8{,}099$ non-trivial questions, both regimes occur: sometimes
\textsc{When} dominates, sometimes \textsc{Whom}.

\subsection{When the signal hurts: \texttt{engineering\_330}}
\label{sec:engineering330}

Not every question favors the normal polarity. On \texttt{engineering\_330}
(gold $=$ B, voter picks split $7 \times$ A, $7 \times$ B, $1 \times$ C,
$1 \times$ Z), the B-pickers have \emph{higher} mean $\sema$ (\,$0.608$\,)
than the A-pickers ($0.517$): the solver is more internally consistent on
the wrong answer. Under $\conf = 1 - \sema$, A-pickers retain more
own-pick mass and A wins. Switching to inverted polarity
($\conf = \sema$) recovers B. This is a ``confidently wrong'' failure of
letter-entropy as a per-question confidence signal, and is the population
that \texttt{inverted\_x\_div} targets. The two modes thus rescue disjoint
families of questions; on MMLU-Pro the normal-polarity family is the
larger of the two.

\subsection{The \textsc{When}/\textsc{Whom} decomposition}

PPV's column construction (Eq.~\ref{eq:column}) has two levers,
\textsc{When} (how much to keep on own pick) and \textsc{Whom} (how to
split the rest). To isolate which lever delivers the gain, we ablate each
independently:
\begin{itemize}[leftmargin=1.2em, topsep=2pt, itemsep=2pt]
\item \textbf{Fix \textsc{Whom}, vary \textsc{When}.} Holding peer weights
  at clipped centered cosine, varying $\conf_j$ over the five modes
  (Table~\ref{tab:modes}) produces the full spread of accuracies in
  Table~\ref{tab:main}.
\item \textbf{Fix \textsc{When}, vary \textsc{Whom}.} Holding
  \textsc{When} at each of the five modes and multiplying
  $\max(\cos_{ij}, 0)$ by each of four peer-quality scores (plus none)
  gives a $5 \times 5$ grid of configurations. No cell beats cosine-only
  \textsc{Whom}, with the best quality variant losing by a single
  question, and $11$ of the $25$ cells lose significantly (paired
  McNemar $p < 0.05$): quality-gating with the wrong polarity is
  actively harmful.
\end{itemize}
\textsc{When} is the load-bearing lever. \textsc{Whom} as plain clipped
cosine is essentially optimal in this design space.

\paragraph{Why \textsc{Whom} doesn't need an explicit quality gate.}
PPV's repeated-squaring propagation handles peer-side quality implicitly. A
low-quality voter that receives mass routes it back out through \emph{its
own column}, which again uses cosine. After two to three hops, mass settles
in neighborhoods of mutually high-cosine voters which, empirically, are
also high-quality neighborhoods.

\subsection{The centering ablation}

Skipping per-question centering, that is, using raw $e_c$ in place of
$\tilde e_c$, is the single largest ablation. Off-diagonal cosines become
near-constant (std $0.025$, range $[+0.88, +0.99]$), the \textsc{Whom}
block degenerates to near-uniform peer weights, and PPV reduces to a soft
averaging that brings no gain over majority. The centering trick is the
difference between PPV-as-aggregator and PPV-as-soft-majority.

% =============================================================================
% Theory: two-block characterization of when PPV overturns majority
% =============================================================================
% =============================================================================
% theory.tex — Two-block characterization of when PPV overturns majority.
% Input from draft.tex between the Mechanism and Negative Results sections.
% Requires in the preamble:
%   \usepackage{amsthm}
%   \newtheorem{theorem}{Theorem}
%   \newtheorem{proposition}{Proposition}
%   \newtheorem{lemma}{Lemma}
%   \newtheorem{corollary}{Corollary}
%   \newtheorem{remark}{Remark}
%   \newtheorem{assumption}{Assumption}
% Numerical checks in this section are reproduced by
%   experiments/theory/verify_two_block.py (closed form, grid, worked example)
%   experiments/theory/boundary_scatter.py  (full-dataset validation + figure
%   via experiments/theory/plot_boundary_scatter.py)
% =============================================================================

\section{When Does Delegation Beat Majority? A Two-Block Characterization}
\label{sec:theory}

The experiments of \S\ref{sec:experiments} answer the distributional
question, showing that PPV beats majority \emph{on average}, but the
mechanism study of \S\ref{sec:mechanism} suggests a sharper, per-instance
question: given
a concrete disagreement between a majority cluster and a minority cluster,
\emph{which one does PPV resolve to, and why?} This section answers that
question exactly in an idealized \emph{two-block} model of the voting
matrix, and validates the idealization at scale: across the full
MMLU-Pro run, the closed form calls the realized winner of the
$16 \times 16$ chain on $96.5\%$ of non-trivial questions and its
predicted mass gap tracks the realized gap at $r = 0.97$
(\S\ref{sec:theory-validation}). The result is a single
inequality that formalizes the narrative of \S\ref{sec:philosophy314}:
delegation overturns a majority precisely when the majority's
\emph{weighted leakage} toward the minority exceeds its raw vote margin.

\subsection{The two-block model}
\label{sec:two-block-model}

Consider a question on which the $n$ voters split into two clusters: a
block $A$ of $k$ voters picking letter $p_A$ and a block $B$ of
$m = n - k$ voters picking letter $p_B$, with $k \ge m$. (Real questions
may spread over more letters; the model addresses the top-two clusters,
which on the non-trivial subset carry the bulk of the mass.) We idealize
the column construction of Eq.~\ref{eq:column} by giving every voter in a
block the same signals:

\begin{assumption}[Block symmetry]
\label{ass:block-symmetry}
Every voter in block $X \in \{A, B\}$ has own-pick confidence
$\alpha_X \in (0, 1]$ and splits its peer budget $1 - \alpha_X$ as
follows: a fraction $\lambda_X \in [0, 1]$ crosses to the other block
(uniformly over its voters) and the rest stays within the block
(uniformly over the $|X| - 1$ same-block peers).\footnote{If a block is a
singleton it has no same-block peer and $\lambda_X = 1$ necessarily.}
\end{assumption}

The parameter $\lambda_X$ is the \emph{leakage} of block $X$: the share
of delegated mass that defects to the opposing cluster. In the
instantiation of \S\ref{sec:method}, $\alpha_X$ is the block's mean
confidence ($1 - \sema$ under the \texttt{confidence} mode) and
$\lambda_X$ summarizes the clipped-cosine geometry: a block whose
reasoning is internally incoherent but geometrically adjacent to a tight
opposing cluster has high leakage, exactly the configuration of the
D-cluster in \S\ref{sec:philosophy314}.

\begin{lemma}[Exact lumping]
\label{lem:lumping}
Under Assumption~\ref{ass:block-symmetry}, the absorption probabilities
of the $n$-voter chain are constant on blocks, and the chain is
equivalent to the four-state chain on $\{A, B, p_A, p_B\}$ whose two
non-absorbing states send their outgoing mass to
\[
\begin{aligned}
A &\;\to\; \alpha_A\, p_A \;+\; t_A\, B \;+\; (1 - \alpha_A - t_A)\, A,\\
B &\;\to\; \alpha_B\, p_B \;+\; t_B\, A \;+\; (1 - \alpha_B - t_B)\, B,
\end{aligned}
\]
while $p_A$ and $p_B$ are absorbing. Here each right-hand side is the
outgoing distribution over destination states, and
$t_X = (1 - \alpha_X)\,\lambda_X$ is the \emph{cross-mass} of block $X$.
\end{lemma}

\begin{proof}
The transition kernel is invariant under the action of
$S_k \times S_m$ permuting voters within blocks, so absorption
probabilities are constant on each block. Aggregating states by block
gives the stated four-state chain: from any voter of $A$, total mass
$\alpha_A$ absorbs at $p_A$, total mass $(1-\alpha_A)\lambda_A = t_A$
crosses to block $B$, and the remainder stays inside block $A$
(distributed among peers, which by constancy of absorption probabilities
may be lumped into a single state); symmetrically for $B$.
\end{proof}

\subsection{Absorbed mass and the flip condition}

\begin{proposition}[Closed-form absorbed mass]
\label{prop:masses}
Under Assumption~\ref{ass:block-symmetry}, the total mass absorbed at
each policy (Eq.~\ref{eq:winner}) is
\begin{equation}
M_{A} = \frac{\alpha_A \bigl[\, k\,(\alpha_B + t_B) + m\, t_B \,\bigr]}{D},
\qquad
M_{B} = n - M_{A},
\label{eq:mass}
\end{equation}
with $D = \alpha_A \alpha_B + \alpha_A t_B + \alpha_B t_A > 0$.
\end{proposition}

\begin{proof}
Let $u_X$ denote the probability that a unit of mass starting at block
$X$ absorbs at $p_A$. First-step analysis on the lumped chain of
Lemma~\ref{lem:lumping} gives
\[
\begin{aligned}
u_A &= \alpha_A + (1 - \alpha_A - t_A)\, u_A + t_A\, u_B,\\
u_B &= (1 - \alpha_B - t_B)\, u_B + t_B\, u_A .
\end{aligned}
\]
Rearranged, $(\alpha_A + t_A)\, u_A = \alpha_A + t_A u_B$ and
$(\alpha_B + t_B)\, u_B = t_B u_A$. Substituting the second into the
first and using the identity
$(\alpha_A + t_A)(\alpha_B + t_B) - t_A t_B = D$ yields
\[
u_A = \frac{\alpha_A (\alpha_B + t_B)}{D},
\qquad
u_B = \frac{\alpha_A\, t_B}{D}.
\]
Since $\alpha_A, \alpha_B > 0$ implies $D > 0$, the chain absorbs almost
surely and $M_A = k u_A + m u_B$ gives Eq.~\ref{eq:mass};
$M_B = n - M_A$ by conservation of the $n$ units of voting mass.
\end{proof}

\begin{theorem}[Flip condition]
\label{thm:flip}
Define the \emph{delegation odds} of block $X$ as
$r_X = \frac{1 - \alpha_X}{\alpha_X}$, the ratio of delegated to retained
mass. Under Assumption~\ref{ass:block-symmetry}, PPV resolves the
question to the minority letter $p_B$ if and only if
\begin{equation}
n \,\bigl( r_A \lambda_A \;-\; r_B \lambda_B \bigr) \;>\; k - m .
\label{eq:flip}
\end{equation}
\end{theorem}

\begin{proof}
By Proposition~\ref{prop:masses}, $M_B > M_A$ iff
$\alpha_B [\, m(\alpha_A + t_A) + k t_A \,] >
 \alpha_A [\, k(\alpha_B + t_B) + m t_B \,]$.
Expanding and collecting,
\[
\alpha_A \alpha_B\, (m - k) \;+\; n\,(\alpha_B t_A - \alpha_A t_B) \;>\; 0 .
\]
Dividing by $\alpha_A \alpha_B > 0$ and substituting
$t_X = (1 - \alpha_X)\lambda_X$ gives Eq.~\ref{eq:flip}.
\end{proof}

Eq.~\ref{eq:flip} is the formal answer to the title question. The
left-hand side is the \emph{weighted leakage asymmetry}: each block's
leakage $\lambda_X$ (a \textsc{Whom} quantity, set by reasoning
geometry) amplified by its delegation odds $r_X$ (a \textsc{When}
quantity, set by letter entropy). The right-hand side is the raw vote
margin that majority voting reports. Delegation beats majority exactly
when the majority holds its votes weakly \emph{and} routes them toward
the minority, while the minority does not reciprocate.

\begin{corollary}[Conditional do-no-harm]
\label{cor:no-harm}
No leakage pattern can overturn the majority when
$k - m \ge n\, r_A$, i.e.\ when
\[
\alpha_A \;\ge\; \frac{n}{\,n + (k - m)\,}.
\]
\end{corollary}

\begin{proof}
$\lambda_A \le 1$ and $\lambda_B \ge 0$ bound the left side of
Eq.~\ref{eq:flip} by $n r_A$.
\end{proof}

For example, a $12$--$4$ majority ($k - m = 8$, $n = 16$) is safe from
any flip whenever the majority's mean confidence is at least
$16/24 = 2/3$; this is the theoretical counterpart of the trivial-subset
threshold of Table~\ref{tab:main}. The safety region is
\emph{conditional} on the confidence floor, and necessarily so:
\citet{kahng2018liquid} prove that no local delegation mechanism can
guarantee unconditional do-no-harm while retaining any gain over direct
voting. Theorem~\ref{thm:flip} sharpens that impossibility into an exact
boundary for this mechanism.

\begin{remark}[Majority recovery]
If the two signals carry no asymmetry, meaning $\alpha_A = \alpha_B$ and
$\lambda_A = \lambda_B$, the left side of Eq.~\ref{eq:flip} vanishes
and PPV agrees with majority for any $k > m$. PPV deviates from majority
only when entropy or geometry actually discriminates between the
clusters.
\end{remark}

\begin{remark}[Tiebreak]
For $k = m$ the condition reduces to $r_A \lambda_A > r_B \lambda_B$:
the tied cluster with the smaller odds-weighted leakage wins. This is
the \texttt{economics\_124} regime of \S\ref{sec:philosophy314}, where a
modest entropy gap ($r$) decides a tie under near-symmetric geometry
($\lambda$).
\end{remark}

\begin{remark}[Polarity]
Nothing in Theorem~\ref{thm:flip} knows which block holds gold. When the
solver is confidently wrong (\texttt{engineering\_330}: gold sits in the
\emph{higher}-entropy block), $r$ works against the gold cluster and the
same inequality correctly predicts that PPV under the
\texttt{confidence} mode loses the question. The theorem thus also
formalizes the polarity failure mode of \S\ref{sec:engineering330}.
\end{remark}

\subsection{Numerical validation}
\label{sec:theory-validation}

We validate the theory in three steps
(\texttt{experiments/theory/verify\_two\_block.py}).

\paragraph{Exactness of the closed form.} We instantiate the explicit
$(n+2) \times (n+2)$ column-stochastic matrix of
Assumption~\ref{ass:block-symmetry} and compute its limit by the
repeated-squaring procedure of \S\ref{sec:prelim}. Across $200$ random
configurations $(k, m, \alpha_A, \alpha_B, \lambda_A, \lambda_B)$ the
absorbed masses match Eq.~\ref{eq:mass} to $1.7 \times 10^{-14}$
(floating-point precision), and across a grid of $19{,}404$
configurations the sign predicted by Eq.~\ref{eq:flip} matches the
brute-force winner in every case.

\paragraph{The worked example, predicted.} We then feed
Theorem~\ref{thm:flip} the block summary statistics of
\texttt{philosophy\_314} (\S\ref{sec:philosophy314}), and nothing
else: $k = 10$, $m = 6$, mean confidences $\alpha_A = 0.287$,
$\alpha_B = 0.327$ (Table~\ref{tab:phil314-voters}), and leakages
$\lambda_A = 0.534$, $\lambda_B = 0.481$ (each block's mean share of
peer budget routed to the other block under clipped centered cosine).
Six numbers summarizing a $16 \times 16$ matrix. The flip condition
reads $16\,(2.49 \times 0.534 - 2.06 \times 0.481) = 5.43 > 4$: the
theorem predicts the flip, and moreover predicts it is \emph{fragile}:
the inequality fails for any $\lambda_B > 0.52$, so had the minority
routed another four percentage points of its peer budget back toward
the majority, the $10$--$6$ margin would have stood. The predicted masses
(Eq.~\ref{eq:mass}) are $M_D = 7.78$, $M_I = 8.22$, against the realized
full-matrix values of $7.06$ and $8.94$ (\S\ref{sec:philosophy314}): the
two-block idealization calls the winner and the direction of the mass
gap correctly while compressing its magnitude, the price of replacing
the heterogeneous per-voter cosine structure with two block averages.
The dataset-level comparison below quantifies this compression: on
average it is small and unbiased.

\paragraph{Full-dataset validation.} We repeat the six-number extraction
on every question of the MMLU-Pro run: top-two letter clusters define
the blocks, $\alpha_X$ is the block mean confidence, and $\lambda_X$ the
block mean cross-routed peer-budget share, exactly as above
(Figure~\ref{fig:boundary-scatter}). On the non-trivial questions whose
full-chain winner is one of the two block letters ($n = 7{,}688$), the
sign of Eq.~\ref{eq:flip} matches the realized outcome on $96.5\%$, with
flip precision $0.88$ and recall $0.84$; on the subset that is exactly
two-block ($n = 769$, no third cluster and no abstentions), accuracy
rises to $97.5\%$ and recall to $0.95$. Mispredictions concentrate where
the two sides of Eq.~\ref{eq:flip} nearly tie
(Figure~\ref{fig:boundary-scatter}, left). The parameter-free boundary
is also at the fitted optimum: a grid search over linear boundaries
$y > ax + b$ in this plane gains less than $0.1$\,pp in-sample, so a
trained classifier has no room to improve on the derived
slope-one, intercept-zero line. Because
Proposition~\ref{prop:masses} predicts \emph{masses} rather than only
the binary winner, the theory is falsifiable at a finer grain: across
the same questions, regressing the realized normalized mass gap on the
predicted one gives Pearson $r = 0.973$ with slope $1.02$
(Figure~\ref{fig:boundary-scatter}, right) --- the two-block closed form
is a nearly unbiased predictor of the full chain's absorbed masses. On
the $2{,}962$ trivial questions ($\ge 12/16$ supermajorities) the
condition predicts no flip and the chain realizes none, the empirical
counterpart of Corollary~\ref{cor:no-harm}. The transfer of this
boundary across solver models and benchmarks is left to the expanded
experimental study.

\begin{figure*}[t]
\centering
\includegraphics[width=0.92\textwidth]{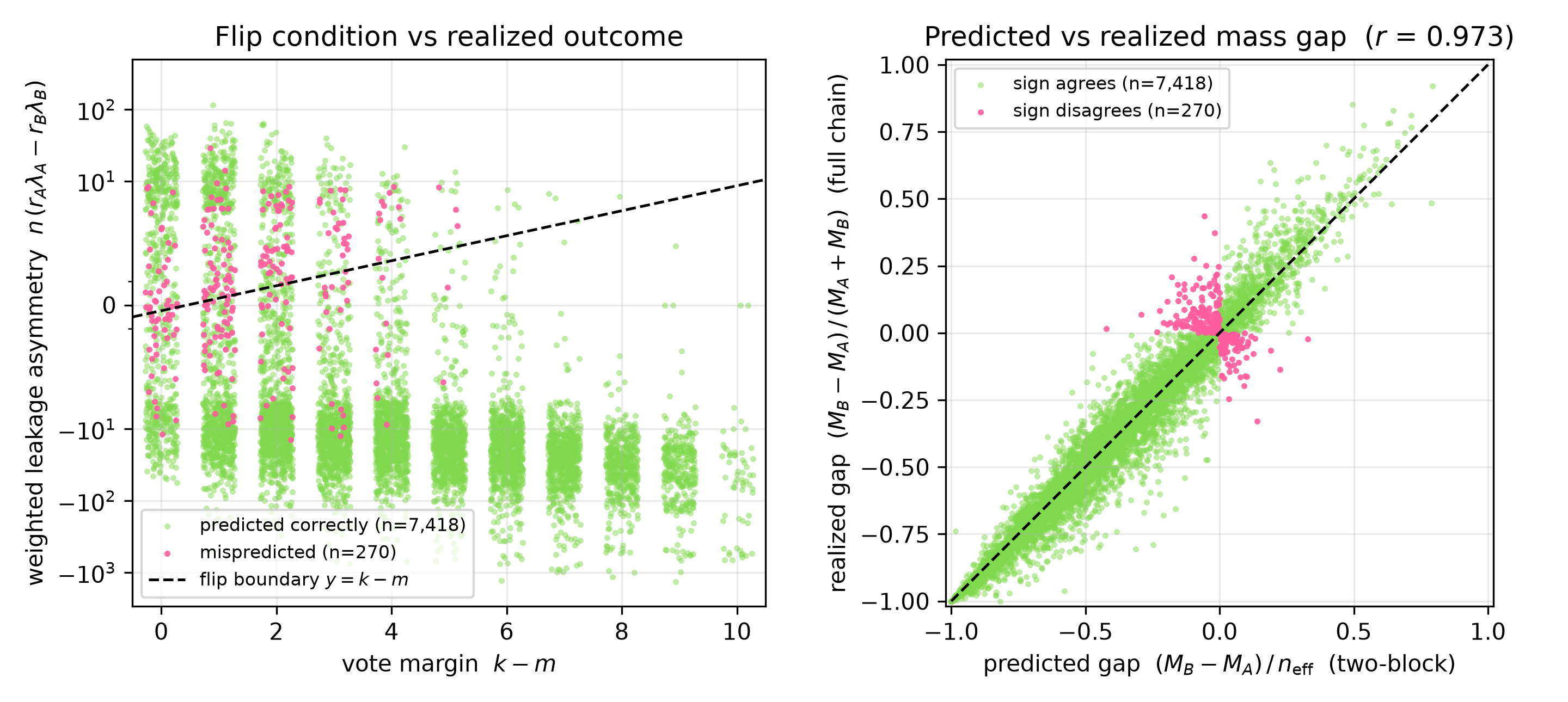}
\caption{\textbf{Full-dataset validation of Theorem~\ref{thm:flip}}
(MMLU-Pro, non-trivial questions resolved to a top-two letter,
$n = 7{,}688$). \emph{Left:} each question placed by its vote margin
$k - m$ and weighted leakage asymmetry $n(r_A \lambda_A - r_B \lambda_B)$
(symlog scale); the dashed line is the flip boundary of
Eq.~\ref{eq:flip}. Green: the closed form predicts the realized outcome
of the full $16$-voter chain; pink: misprediction ($3.5\%$),
concentrated at the boundary. \emph{Right:} predicted (two-block)
versus realized (full-chain) normalized mass gap; Pearson $r = 0.973$,
OLS slope $1.02$. Sign disagreements (pink) concentrate near the origin,
where both gaps are close to zero.}
\label{fig:boundary-scatter}
\end{figure*}

% =============================================================================
\section{Negative Results}
\label{sec:negatives}
% =============================================================================

\subsection{$P(\text{True})$ is anti-correlated with correctness}

The same-model $P(\text{True})$ signal of \citet{kadavath2022know} is a
standard auxiliary confidence. We computed $P(\text{True})$ per voter over
all $n \cdot |Q| = 192{,}512$ (voter, question) rows on this setup
(Table~\ref{tab:confidence}). Its AUROC for predicting correctness is
$0.47$, which is below chance. The top-$25\%$ of voters by $P(\text{True})$ are
correct $32.1\%$ of the time, below the base rate of $35.4\%$. CoCoA-style
products with $\sema$ inherit the anti-correlation and underperform
$\sema$ alone. We read this as model- and setting-specific:
high-temperature generation from a $1.7$B-parameter solver produces
confident-sounding-but-wrong outputs and its same-model verifier rewards
them. For this setup the consequence is clear: $P(\text{True})$ should not
enter the aggregator.

\subsection{Unsupervised mode selection: a pilot}
\label{sec:ensemble-pilot}

The five confidence modes rescue overlapping but distinct question
families (\S\ref{sec:experiments}), so a per-question mode selector
could in principle approach the pass@16 ceiling. On a $100$-question
random sample we piloted two unsupervised selectors:
\texttt{ensemble\_margin} (run two modes, keep the one whose PPV
consensus margin is sharper) and \texttt{ensemble\_agree\_else\_maj}
(if the two modes agree, take that answer; otherwise fall back to
majority). Neither beat the best single mode: \texttt{ensemble\_margin}
matched \texttt{confidence} alone ($+3$ questions over majority), the
agreement fallback did worse ($+2$), while a gold-aware oracle selector
reached $+6$. We read this narrowly: the selectors we tried do not
close the gap. Whether better \textsc{When}/\textsc{Whom} signals or
selectors exist is open, since our letter entropy and centered cosine are
one instantiation of the levers rather than the ceiling of the framework.
Theorem~\ref{thm:flip} states precisely what any candidate must
deliver: the correct polarity of the leakage asymmetry, per question.

\begin{table}[t]
\centering
\caption{Per-voter confidence calibration on MMLU-Pro (Qwen3-1.7B,
$192{,}512$ voter-question rows). Base rate $= 0.354$.}
\label{tab:confidence}
\small
\begin{tabular}{lrrr}
\toprule
\textbf{Signal} & \textbf{AUROC} & \textbf{Top-25\%} & \textbf{Top-50\%} \\
\midrule
$1 - \sema_{\mathrm{plug}}$       & $0.705$ & $0.588$ & $0.486$ \\
$1 - \sema_{\mathrm{mm}}$         & $0.704$ & $0.586$ & $0.486$ \\
$P(\text{True}) \cdot \sema$      & $0.649$ & $0.549$ & $0.455$ \\
mean $P(\text{True})$ over 8      & $0.614$ & $0.514$ & $0.427$ \\
\textbf{$P(\text{True})$ (raw)}   & \textbf{$0.472$} & \textbf{$0.321$} & \textbf{$0.328$} \\
\bottomrule
\end{tabular}
\end{table}

% =============================================================================
\section{Discussion}
% =============================================================================

\paragraph{The signal: letter entropy is load-bearing, geometry is the medium.}
Across all our experiments the dominant ablation is the choice of
\textsc{When}. Geometry (\textsc{Whom}) is necessary for the propagation
to flow non-trivially, but its functional form is essentially fixed once
centered cosine is in place. This pattern of having a strong per-voter scalar
signal channeled through a mild geometric backbone suggests a general
design principle for delegation-based aggregators: invest in the scalar,
use the embedding only to route.

\paragraph{Why propagation, not weighting?} A simpler approach would
weight each voter's pick by $\conf_j$ and sum: an entropy-weighted
majority. PPV's propagation differs in that low-confidence voters do not
just contribute less; they \emph{redistribute} their voting mass to
peers. On \texttt{philosophy\_314} (\S\ref{sec:philosophy314}) this
redistribution is what overturns the majority: the D-cluster's leaked
budget, routed by clipped cosine, concentrates on the geometrically
coherent I-cluster --- D-pickers route over half their peer budget across
because the within-D weights are too small to compete --- and propagation
compounds the asymmetry over multiple hops.

\paragraph{Polarity is question-dependent.} The cleanest open problem we
leave is the polarity question: on most MMLU-Pro questions, lower-entropy
voters are more reliable, but on a structured minority the relationship
inverts. Detecting this per question without gold labels is, on the
features we tried, unreliable. A supervised polarity classifier may be 
the next step, but it changes the regime: aggregation becomes
``unsupervised at inference time but supervised at design time.''

\paragraph{Knowing \textsc{When} and \textsc{Whom} as a capability.}
Theorem~\ref{thm:flip} can also be read as a demand on the model rather
than on the aggregator: delegation helps exactly when the sampled
population emits \textsc{When} signals (self-confidence) and
\textsc{Whom} signals (peer recognition) whose joint asymmetry points
toward the correct cluster, and Corollary~\ref{cor:no-harm} says harm is
avoided only when confidence is high precisely where the majority is
right. This is a stronger requirement than probability calibration,
because it is \emph{relational}: the model must know when it knows
\emph{and} recognize which of its peers reason soundly, jointly and per
instance. Because the boundary in Eq.~\ref{eq:flip} is closed-form, the
requirement is directly measurable: score a model by the fraction of
instances whose realized $(k, m, \conf, \lambda)$ configuration falls on
the gold side of the flip condition. Our negative results show current
signals deliver this only partially: $P(\text{True})$ is anti-correlated
with correctness, and polarity inverts on a structured minority of
questions. We therefore see \emph{delegation competence}, the capacity
to give away one's vote without doing harm, as a candidate axis of
capability distinct from raw accuracy, one that future models could be
evaluated on, or trained toward, directly.

% =============================================================================
\section{Conclusion}
% =============================================================================

Self-consistency leaves two free signals on the table. By feeding
letter-level semantic entropy and centered embedding cosine into PPV
\citep{sakai2025ppv}, we obtain an unsupervised aggregator that beats
majority by $+2.24$\,pp on the non-trivial subset of MMLU-Pro at scale,
with paired McNemar significance $p \approx 10^{-14}$. The
\textsc{When}/\textsc{Whom} decomposition shows the gain is delivered by
per-voter entropy as confidence; the geometric \textsc{Whom} side stays
cosine-only because PPV's propagation implicitly launders peer-side
quality. A two-block analysis makes the title question exact: delegation
overturns a $k$--$m$ majority precisely when the odds-weighted leakage
asymmetry exceeds the vote margin (Theorem~\ref{thm:flip}), with
do-no-harm guaranteed above a closed-form confidence floor. Negative
results constrain the design space for future
unsupervised LLM aggregators and isolate the open problem (per-question
polarity selection) where supervision plausibly helps.

% =============================================================================
% Bibliography
% =============================================================================
\onecolumn
\clearpage
\bibliographystyle{plainnat}
\bibliography{references}

% --- The entries below are kept as a fallback thebibliography block.
% --- Delete from here to \end{thebibliography} if you use references.bib.
\begin{comment}

\end{comment}

\end{document}